%% file: main.tex
\documentclass[sigconf,nonacm]{acmart}
\usepackage{algpseudocode}
\usepackage{xcolor}
\usepackage{subcaption}
\usepackage{amsthm}
\newtheorem{Lemma}{Lemma}
\usepackage{algorithm}
\usepackage{algpseudocode}
\usepackage{titlesec}
\titlespacing{\subsubsection}{0pt}{0.7\baselineskip}{0.3\baselineskip}

\acmConference[GLSVLSI'25]{Great Lakes Symposium on VLSI (GLSVLSI) 2025}{June 30-- July 02,
  2025}{New Orleans, LA}

\begin{document}

\title{DPQ-HD: Post-Training Compression for Ultra-Low Power Hyperdimensional Computing}

\author{Nilesh Prasad Pandey}
\email{nppandey@ucsd.edu}
\affiliation{%
  \institution{University of California San Diego}
  \city{San Diego}
  \state{California}
  \country{USA}
}

\author{Shriniwas Kulkarni}
\email{s7kulkarni@ucsd.edu}
\affiliation{%
  \institution{University of California San Diego}
  \city{San Diego}
  \state{California}
  \country{USA}
}

\author{David Wang}
\email{dyw001@ucsd.edu}
\affiliation{%
  \institution{University of California San Diego}
  \city{San Diego}
  \state{California}
  \country{USA}
}

\author{Onat Gungor}
\email{ogungor@ucsd.edu}
\affiliation{%
  \institution{University of California San Diego}
  \city{San Diego}
  \state{California}
  \country{USA}
}

\author{Flavio Ponzina}
\email{fponzina@ucsd.edu}
\affiliation{%
  \institution{University of California San Diego}
  \city{San Diego}
  \state{California}
  \country{USA}
}
\author{Tajana Rosing}
\email{tajana@ucsd.edu}
\affiliation{%
  \institution{University of California San Diego}
  \city{San Diego}
  \state{California}
  \country{USA}
}


\begin{abstract}
Hyperdimensional Computing (HDC) is emerging as a promising approach for edge AI, offering a balance between accuracy and efficiency. However, current HDC-based applications often rely on high-precision models and/or encoding matrices to achieve competitive performance, which imposes significant computational and memory demands, especially for ultra-low power devices. While recent efforts use techniques like precision reduction and pruning to increase the efficiency, most require retraining to maintain performance, making them expensive and impractical. To address this issue, we propose a novel Post Training Compression algorithm, Decomposition-Pruning-Quantization (DPQ-HD), which aims at compressing the end-to-end HDC system, achieving near floating point performance without the need of retraining. DPQ-HD reduces computational and memory overhead by uniquely combining the above three compression techniques and efficiently adapts to hardware constraints. Additionally, we introduce an energy-efficient inference approach that progressively evaluates similarity scores such as cosine similarity and performs early exit to reduce the computation, accelerating prediction inference while maintaining accuracy. We demonstrate that DPQ-HD achieves up to 20-100$\times$ reduction in memory for image and graph classification tasks with only a 1-2\% drop in accuracy compared to uncompressed workloads. Lastly, we show that DPQ-HD outperforms the existing post-training compression methods and performs better or at par with retraining-based state-of-the-art techniques, requiring significantly less overall optimization time (up to 100$\times$) and faster inference (up to 56$\times$) on a microcontroller.
\end{abstract}






\maketitle
\input{sections/introduction}

\input{sections/related}

\input{sections/proposal}

\input{sections/results}

\input{sections/conclusions}

\input{sections/acknowledgement}
\bibliographystyle{ACM-Reference-Format}
\bibliography{main}

\end{document}

%% file: sections/introduction.tex
\section{Introduction}\label{sec:intro}
The integration of artificial intelligence into edge devices is rapidly gaining traction, driven by the increasing demand for real-time, efficient, and privacy-preserving solutions across diverse scientific and industrial applications~\cite{singh2023edge}. However, deploying AI models on resource-constrained systems poses significant challenges, particularly in balancing computational efficiency, memory usage, energy consumption, and model accuracy. Achieving this balance is critical for ensuring low-latency and high-performance AI at the edge.
Hyperdimensional Computing (HDC), a brain-inspired machine learning paradigm, has emerged as a promising approach for edge AI due to its lightweight and highly parallelizable nature~\cite{kanerva2009hyperdimensional}. HDC maps input data into a high-dimensional space, where classification and regression tasks rely on simple element-wise operations, eliminating the need for backpropagation-based training. This inherent efficiency and scalability make HDC an attractive candidate for energy-efficient edge AI~\cite{chang2023recent,amrouch2022brain,behnam2021tiny,khaleghi2022generic}. 

Despite these advantages, optimizing HDC workloads further remains crucial for enabling edge AI on ultra-low-power embedded systems~\cite{imani2019quanthd,ponzina2024microhd,xu2023fsl}. Techniques such as quantization and dimensionality reduction have been widely explored to enhance computational and memory efficiency by reducing data precision and hyperspace size. However, these methods often require re-training to recover accuracy, imposing additional constraints on training data availability and quality. This limitation renders such methods impractical for real-world scenarios where labeled data is scarce or absent, such as in edge IoT systems trained on continuous streaming data. Furthermore, existing approaches frequently focus on isolated optimizations, such as quantization or pruning, thus limiting their potential for achieving aggressive compression which would be pivotal for HDC deployment on ultra-low power devices. 

To address these challenges, we propose DPQ-HD, a novel post-training compression framework designed to optimize HDC workloads for ultra-low power edge AI. By systematically leveraging low-rank matrix Decomposition (D), Pruning (P), and Quantization (Q), DPQ-HD compresses both the encoding process and the HDC model without requiring retraining. Our key contributions are:.
\begin{itemize}
\item We introduce DPQ-HD, a comprehensive post-training compression framework tailored for ultra-low-power edge AI applications. DPQ-HD optimizes the encoding process and the HDC model to achieve end-to-end efficiency.
\item DPQ-HD applies decomposition, pruning, and quantization to systematically compress the HDC pipeline, significantly reducing memory and computational overhead while maintaining near-floating-point accuracy. Extensive experiments on diverse datasets demonstrate that DPQ-HD achieves up to 20-100× total memory reduction when compared to uncompressed HDC workloads with only a 1-2\% accuracy drop across various image and graph-based applications.
\item To enhance efficiency beyond compression, we introduce a progressive inference approach that dynamically adjusts the processed dimensions, reducing runtime and energy usage. By using early exit, our strategy improves prediction runtime by up to 76.94\% without sacrificing accuracy, complementing model compression with runtime optimization.
\item We show that DPQ-HD outperforms existing post-training compression baselines and performs better or at par with state-of-the-art re-training methods, with significantly reduced optimization time (up to 100×) and substantial inference speedups and lower power consumption (up to 56×) on ultra-low power microcontrollers.
\end{itemize}

%% file: sections/related.tex
\section{Background and Related Work}\label{sec:background}

\subsection{Hyper-Dimensional Computing (HDC)}
HDC is an efficient and brain-inspired computing paradigm that leverages the distributed and holistic properties of high-dimensional spaces to create robust representations and enable highly parallel inference and training operations~\cite{kanerva2009hyperdimensional}. In both the training and inference stages, input data $\mathbf{x} \in \mathcal{X}$ is mapped to a high-dimensional space using an embedding function $\phi(\mathbf{x})$, which transforms inputs from $\mathbb{R}^{n}$ to $\mathbb{R}^{D}$, where $D \gg n$. Among the various encoding schemes, random projection encoding has gained significant attention due to its simplicity, accuracy, and reliance on randomly generated projection matrices~\cite{thomas2021theoretical,chang2023recent,ponzina2024microhd,gungor2024hd}. The training stage involves accumulating encoded hypervectors through aggregation or weighted schemes~\cite{hernandez2021onlinehd}. These class hypervectors represent learned patterns within the data. During inference, test data undergoes the same encoding transformation, and its encoded representation is compared to the stored class hypervectors using a similarity metric to make predictions. The simplicity and inherent parallelism of HDC's operations make it an ideal candidate for edge AI applications requiring low latency and energy efficiency~\cite{chang2023recent,yu2023fully,amrouch2022brain}. 

\subsection{Memory and Compute Demands of HDC}
The memory and computational complexity of HDC models are primarily influenced by three components: the projection matrix ($\mathbf{P}$), the encoded hypervector ($\mathbf{Q}$), and the HDC model ($\mathbf{W}$), which consists of a set of class hypervectors. The encoding stage, often the most resource-intensive, can account for more than than 80\% of an HDC model's memory and runtime requirements. Specifically, the projection matrix $\mathbf{P} \in \mathbb{R}^{F \times D}$ scales with the number of input features ($F$) and the dimensionality ($D$). The HDC model $\mathbf{W} \in \mathbb{R}^{C \times D}$ depends on $D$ and the number of classes ($C$) and in typical HDC implementations, high precision is maintained for both the projection matrix and the model, with $D$ often set to 10k~\cite{ponzina2024microhd,chang2023recent}, ultimately resulting in significant memory and computational demands. Although existing optimization techniques often binarize or quantize encoded hypervectors and models, they frequently overlook the projection matrix, leaving room for improvement. This work addresses these limitations by proposing a holistic compression approach targeting the entire HDC pipeline, optimizing both memory and computational efficiency for ultra-low-power edge AI applications.

\subsection{Related Work}
In recent years, multiple efforts have enhanced the performance of HDC for various applications~\cite{nunes2022graphhd,dutta2022hdnn,amrouch2022brain}. Many methods optimize efficiency through hypervector binarization or lower-precision HDC model weights. However, most prior works focus on quantizing the model and encoded hypervectors, leaving the encoding stage in high precision~\cite{nunes2022graphhd,ponzina2024microhd,thomas2021theoretical}. However, as discussed earlier, key components like the projection matrix and HDC weights consume significant on-device memory, and compressing these offers substantial potential to enhance HDC efficiency.

Among existing methods, QuantHD~\cite{imani2019quanthd} proposes a quantization framework to reduce the precision of input data and the HDC model, achieving computational and memory savings. However, its binary or ternary quantization offers limited benefit on microcontrollers (MCUs) due to their 8-bit computation constraints. Additionally, QuantHD requires retraining to recover performance after compression, making it expensive and impractical for real-world applications. MicroHD~\cite{ponzina2024microhd} improves HDC efficiency for edge devices by adopting an accuracy-driven approach to achieve highly compressed models with less than 1\% accuracy loss. However, it requires iterative optimization, retraining the model from scratch at each step, resulting in significant retraining overhead. Similarly, FSL-HD~\cite{xu2023fsl}, referred to as DeMAT in this work, uses Kronecker product-based decomposition to reduce encoding overhead. Unlike our framework, DeMAT focuses only on encoding and relies on specialized hardware for efficient implementation, limiting its applicability to ultra-low power devices lacking such hardware support. In the post training regime, a recent work, Eff-SparseHD~\cite{buelagala2023energy}, applies redundancy pruning on the HDC model without retraining. However, Eff-SparseHD only prunes HDC model dimensions, overlooking opportunities to optimize the full workload.

In this work, we propose DPQ-HD, a post-training compression framework that efficiently compresses both the encoding process and the HDC model, unlike prior works that typically target only one component to achieve end-to-end efficiency. By compressing both, DPQ-HD significantly enhances the overall efficiency of HDC workloads, achieving up to 20–100× memory reduction with just a 1–2\% accuracy drop on various tasks. It outperforms existing post-training pruning baselines and delivers performance comparable to retraining-based SOTA techniques, all with reduced optimization time and faster microcontroller inference.

%% file: sections/proposal.tex
\section{DPQ-HD Framework}\label{sec:proposal}
\subsection{DPQ-HD: A Post Training Compression Framework}
Figure~\ref{fig:DPQHD-workflow} shows our DPQ-HD framework which makes synergic use of matrix decomposition, pruning, and quantization to significantly reduce the compute and memory overhead of HDC pipelines and includes a novel online inference optimization strategy to further improve inference efficiency. 
First, DPQ-HD replaces the projection matrix with lower-rank components, a step motivated by the inherent smoothness and randomness of full‑precision random projection matrices. Next, it prunes both the projection matrix and HDC hypervectors to further reduce complexity. Finally, it uses quantization to further compress the HDC model and increase efficiency.
The specific order of decomposition, pruning, and quantization was chosen since decomposition preserves essential structural information, thereby ensuring that pruning and quantization are performed on a more compact and well-conditioned representation afterwards. In Section~\ref{justification-sq} we provide a theoretical justification for our chosen compression order after decomposition, demonstrating that applying pruning prior to quantization does not introduce additional error beyond the sum of their individual contributions.

\begin{figure}[t]
   \includegraphics[width=\linewidth]{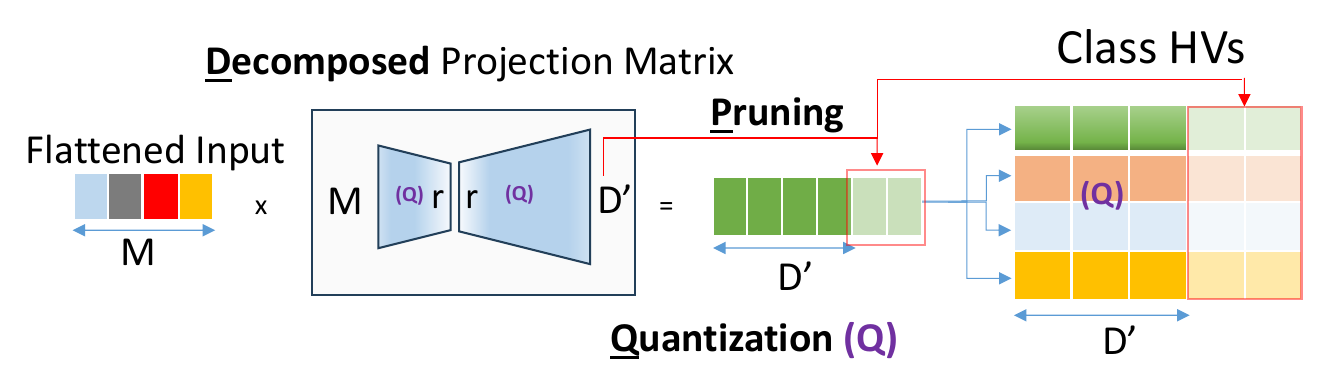}
   \caption{Illustration of DPQ-HD highlighting decomposition, pruning, and quantization. It compresses both the projection matrix and HDC weights for end-to-end efficiency, unlike methods targeting isolated components.}
   \label{fig:DPQHD-workflow}
\end{figure}

\subsubsection{Low Rank Decomposition of Projection Matrix}

As discussed in the previous sections, the projection matrix is a major contributor to memory usage, consuming a significant portion of the available memory resources. In order to improve the efficiency of the encoding process, we propose to use a two level low rank decomposition~\cite{kishore2017literature} of the projection matrix (see Figure \ref{fig:DPQHD-workflow}). This decomposition reduces both the memory and compute requirements of the encoding process. Specifically, the original projection matrix \(\mathbf{P} \in \mathbb{R}^{F \times D}\) is replaced by two smaller randomly initialized matrices: \(\mathbf{P_1} \in \mathbb{R}^{F \times r}\) and \(\mathbf{P_2} \in \mathbb{R}^{r \times D'}\), where \(r\) is much smaller than both \(F\) and \(D\), effectively using a projection matrix \(\mathbf{P'}\) as:

\begin{equation}
\mathbf{P'} \approx \mathbf{P_1} \cdot \mathbf{P_2}
\end{equation}

Now, in contrast to the original encoding process, where input data \(\mathbf{x} \in \mathbb{R}^F\) is projected into a high-dimensional space using \(\mathbf{P}\) to generate a hypervector \(\mathbf{h} \in \mathbb{R}^D\):

\begin{equation}
\mathbf{h} = \mathbf{P} \cdot \mathbf{x}, \quad \mathbf{h} \in \mathbb{R}^D
\end{equation}


With two-level decomposition, encoding is performed in two steps: the input \( \mathbf{x} \) is first transformed by \( \mathbf{P_1} \) to produce \( \mathbf{h_1} \in \mathbb{R}^r \), and then \( \mathbf{h} = \mathbf{P_2} \cdot \mathbf{h_1} = \mathbf{P_2} \cdot (\mathbf{P_1} \cdot \mathbf{x}) \in \mathbb{R}^{D'} \) gives the final high-dimensional representation. This is especially relevant for HDC, where high-dimensional projection matrices incur significant memory and compute costs. Low-rank approximation reduces storage and lowers the number of MACs in encoding, enabling more efficient HDC systems.



\subsubsection{Pruning}

    

To further optimize memory and computational efficiency, we adopt pruning by controlling the \(D'\) dimension in our decomposed encoding matrix. The optimal \(D'\) is selected through a calibration phase, which evaluates pruning impact on accuracy using a small validation set, e.g. 128 samples or less, to minimize accuracy loss. Once the optimal \(D'\) is identified, we adjust the dimensionality of the hypervectors and HDC model weights by removing the last \((D - D')\) dimensions, effectively reducing memory and computation demands across the HDC pipeline. This reduction in dimensionality enables more efficient memory usage and lowers computational costs, making the HDC pipeline suitable for deployment in resource-constrained environments.

\subsubsection{Quantization}
Quantization is a widely adopted technique in both small-scale~\cite{imani2017voicehd,ponzina2024microhd,hernandez2021onlinehd,hernandez2024optimizing} and large-scale~\cite{nagel2021white,pandey2023practical,lin2024awq, pandey2023softmax} machine learning systems, aimed at reducing memory usage and computational overhead. By representing high-precision matrices in lower-precision formats, quantization effectively decreases the storage and processing demands of these matrices. Any high-precision matrix, $W^{r}$, can be approximated through quantization as follows:

    

\begin{equation}
    W_{int} = \text{Clip} \left(  \Big \lfloor{\frac{W^{r}}{s}}\Big \rceil, \text{min}, \text{max} \right), \quad
    W^{r'} = s W_{int}
    \label{eq:quant_dequant}
\end{equation}

\begin{algorithm}[t]
\caption{MSE-Based Post-Training Quantization}
\label{PTQ-algo}
\begin{algorithmic}
\Require Tensor to quantize $T$, Bitwidth $b$
\Ensure Quantized tensor $T_{\text{best}}$, Optimal scale $s_{\text{best}}$

\State \textbf{Initialize:}
\State $t_{\text{max}} \leftarrow \max(|T|)$
\State $q_{\text{max}} \leftarrow 2^{(b-1)} - 1$
\State $s \leftarrow \frac{t_{\text{max}}}{q_{\text{max}}}$
\State $S_{\text{cand}} \leftarrow [0.1s, 0.2s, \dots, 0.9s, s]$
\State $s_{\text{best}} \leftarrow \text{None}$, $\epsilon_{\text{min}} \leftarrow \infty$, $T_{\text{best}} \leftarrow \text{None}$

\For{$s_{\text{cand}} \in S_{\text{cand}}$}
    \State $T_q \leftarrow \text{Quantize}(T, s_{\text{cand}}, - q_{\text{max}}, q_{\text{max}})$ \Comment{Equation~\ref{eq:quant_dequant}}
    \State $T_d \leftarrow \text{Dequantize}(T_q, s_{\text{cand}})$ \Comment{Equation~\ref{eq:quant_dequant}}
    \State $\epsilon \leftarrow \text{MSE}(T_d, T)$
    
    \If{$\epsilon < \epsilon_{\text{min}}$}
        \State $\epsilon_{\text{min}} \leftarrow \epsilon$
        \State $s_{\text{best}} \leftarrow s_{\text{cand}}$
        \State $T_{\text{best}} \leftarrow T_q$
    \EndIf
\EndFor

\State \Return $T_{\text{best}}, s_{\text{best}}$
\end{algorithmic}
\end{algorithm}

where \( W_{\text{int}} \) denotes the low-precision form of \( W^r \), with scale \( s \) and clipping thresholds \( \text{min} \) and \( \text{max} \). We use symmetric quantization, setting \( \text{min} = -2^{(b-1)} \) and \( \text{max} = 2^{(b-1)} - 1 \) based on the bitwidth \( b \).

While quantization is effective, naively reducing high-precision matrices to very low bitwidths can introduce noise. Finding the optimal scale factor is crucial, as it balances resolution and clipping. Our post-training quantization algorithm (Algorithm~\ref{PTQ-algo}) uses a Mean Squared Error (MSE) based approach to select the optimal scale, enabling accurate representation within the limited bitwidth.

\begin{figure*}[t]
    \centering
    \begin{subfigure}[b]{0.40\textwidth}
        \centering
        \includegraphics[width=0.68\textwidth]{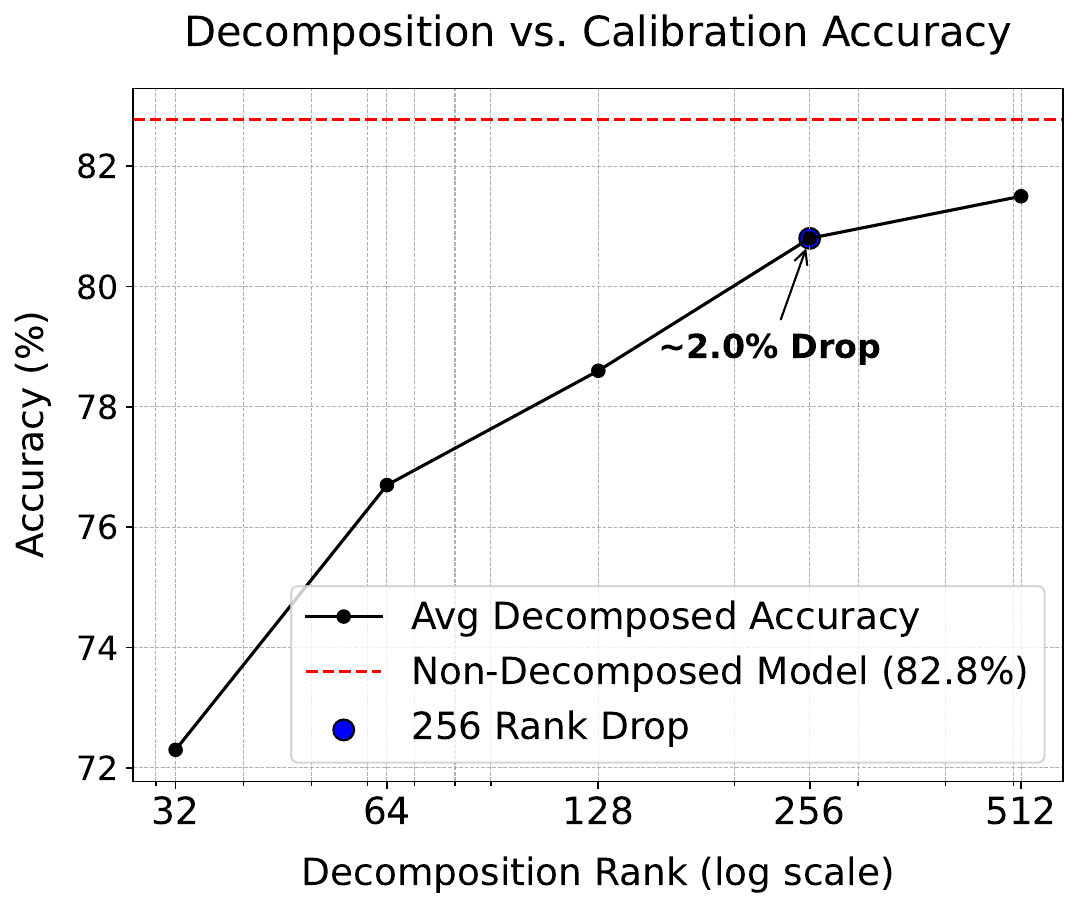}
        \caption{MNIST}
        \label{fig:decomposition_effects}
    \end{subfigure}
    \hfill
    \begin{subfigure}[b]{0.58\textwidth}
        \centering
        \begin{subfigure}[b]{0.48\textwidth} 
            \centering
            \includegraphics[width=\textwidth]{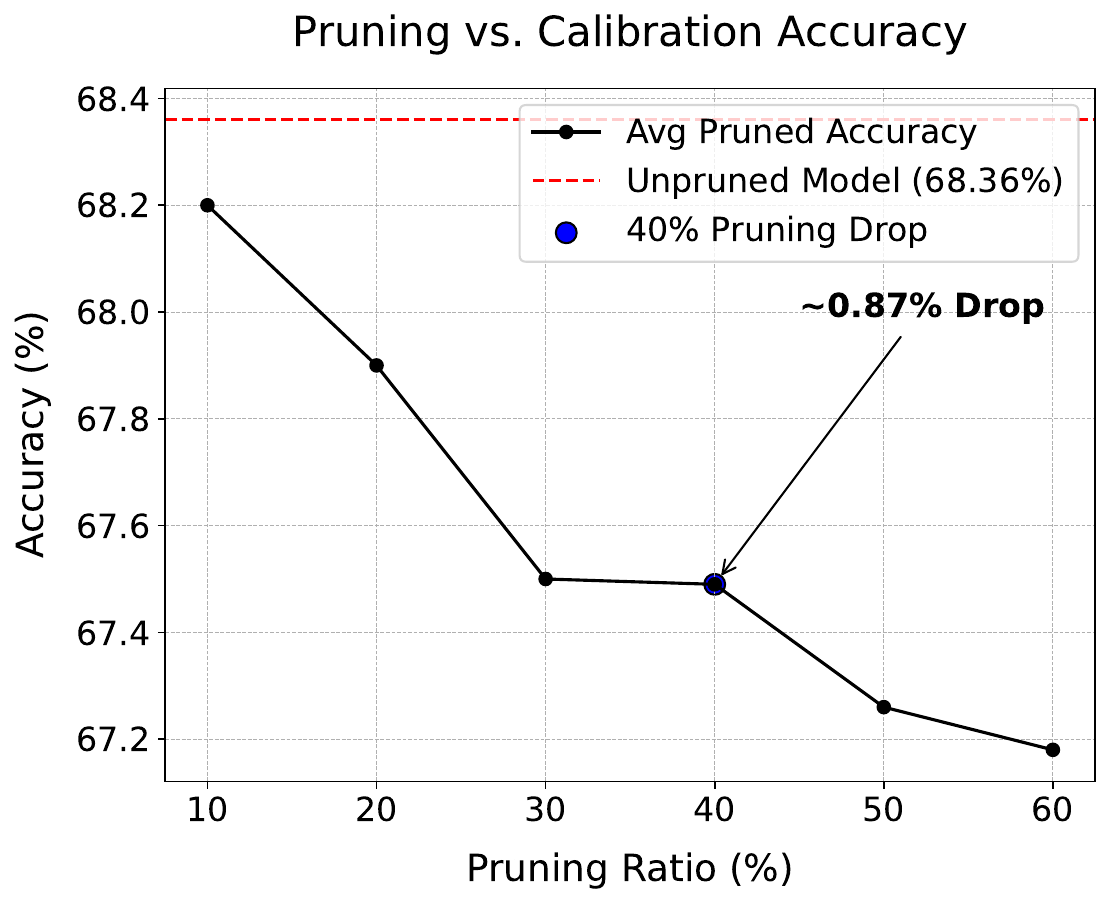}
            \caption*{(b-i) DD}
            \label{fig:dd_pruning}
        \end{subfigure}
        \hfill
        \begin{subfigure}[b]{0.48\textwidth} 
            \centering
            \includegraphics[width=\textwidth]{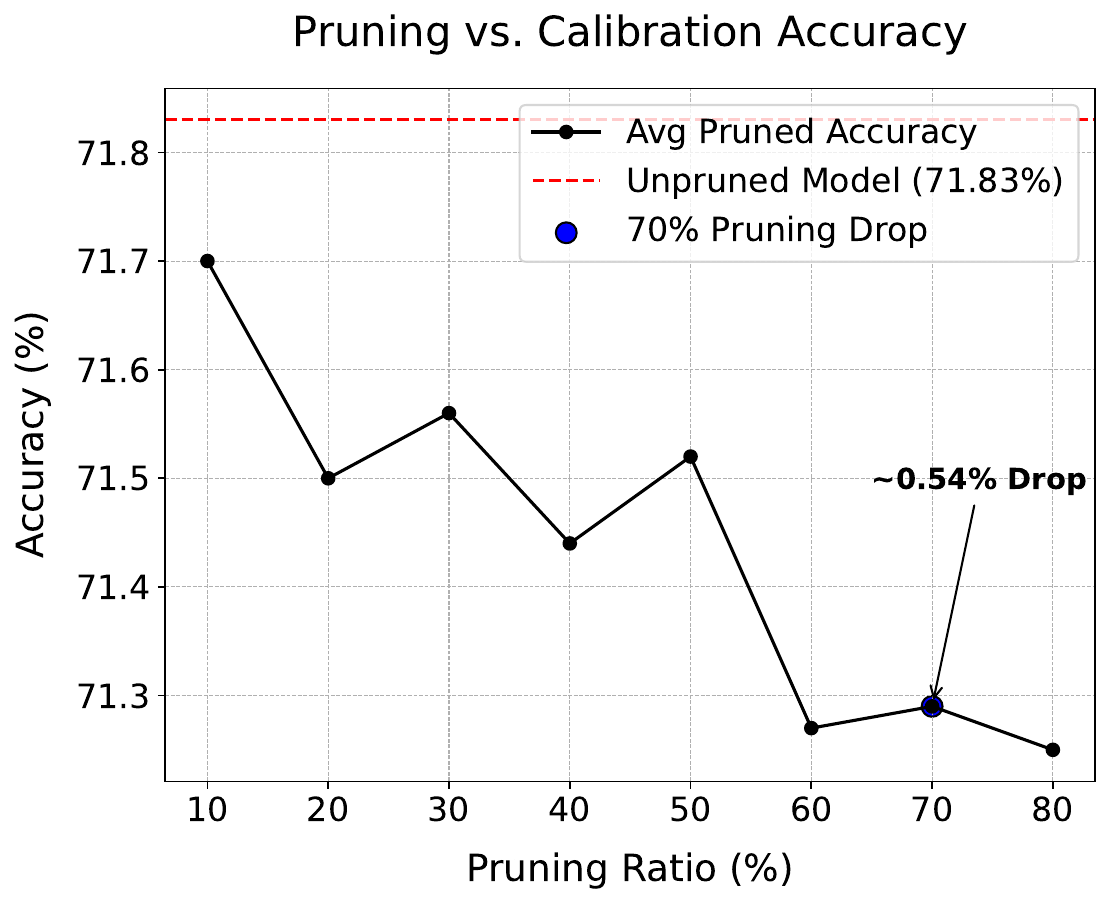}
            \caption*{(b-ii) Fashion-MNIST}
            \label{fig:fashion_mnist_pruning}
        \end{subfigure}
        \label{fig:pruning_effects}
    \end{subfigure}
    \caption{Effect of decomposition rank on calibration accuracy for (a) MNIST and pruning ratio for (b-i) DD and (b-ii) Fashion-MNIST. Accuracy is averaged over 5 subsets of 128 samples.}
    \label{fig:decomposition_pruning_comparison}
\end{figure*}

\subsection{Theoretical Insights on Pruning Before Quantization} \label{justification-sq}
Prior research has explored quantization~\cite{imani2019quanthd} and pruning~\cite{buelagala2023energy} as independent techniques for compressing HDC systems. However, only a limited number of works~\cite{ponzina2024microhd} have investigated their combined effects empirically. In this section, we provide a formal justification for why pruning should be applied before quantization. Specifically, we establish in Lemma~\ref{theorem-1} that applying pruning before quantization does not introduce additional error beyond the cumulative effect of each operation.

\begin{Lemma} \label{theorem-1}
Let \( q \) represent the channel-wise quantization operation and \( s \) denote the pruning transformation, which drops dimensions from the last, similar to the scheme adopted in our work. The error introduced by applying pruning before quantization is no greater than the sum of the individual errors of quantization and pruning. Formally, for any vector \( x \in \mathbb{R}^n \), we define:
\[
\epsilon_{q\circ s}(x) = x - q\bigl(s(x)\bigr),
\quad
\epsilon_q(x) = x - q(x),
\quad
\epsilon_s(x) = x - s(x).
\]
Then, applying \( s \) before \( q \) ensures that the total error remains bounded by the sum of individual errors:
\[
\|\epsilon_{q \circ s}(x)\|
\;\leq\;
\|\epsilon_q(x)\|
\;+\;
\|\epsilon_s(x)\|.
\]
\end{Lemma}

\begin{proof}   
Let the original vector \( x \) be represented as \( [v_1, v_2] \), where \( v_1 \) and \( v_2 \) are its sub-vectors. After pruning, the modified vector becomes \( [v_1, 0] \), where the trailing dimensions are removed, as previously discussed as our pruning strategy.
 Then
\begin{align}
\|\epsilon_{q \circ s}(x)\|
&= \|[v_1,\,v_2] - q([v_1,\,0])\| \notag \\
&= \|[\,v_1 - q(v_1),\,v_2\,]\| \notag \\
&\le \|[\,\epsilon_q(v_1),\,0\,]\| + \|[\,0,\,v_2\,]\| \notag \\
&= \|[\,\epsilon_q(v_1),\,0\,]\| + \|[\,\epsilon_s(v_1),\,\epsilon_s(v_2)\,]\| \notag \\
&\text{(Pruning: $\epsilon_s(v_1) = 0$, $\epsilon_s(v_2) = v_2$)} \notag \\ 
&\le \|[\,\epsilon_q(v_1),\,\epsilon_q(v_2)\,]\| + \|[\,\epsilon_s(v_1),\,\epsilon_s(v_2)\,]\| \notag \\
&= \|\epsilon_q(x)\| + \|\epsilon_s(x)\|
\end{align}
\end{proof}
This result provides a theoretical foundation for structuring compression pipelines in HDC systems, demonstrating that pruning before quantization does not introduce any excess error beyond their individual contributions. While presented here for vectors, this proof naturally extends to matrices, where quantization operates independently on each channel, and pruning removes dimensions from the end.

\subsection{DPQ-HD on Ultra-Low Power Edge AI}
In this work, we focus on deploying AI on edge microcontrollers (MCUs), which are key targets for resource-constrained environments because of their affordability and energy efficiency, as widely explored in previous works~\cite{sridhara2011ultra,chien2016low}. Nonetheless, their limited processing power and memory capacity present significant challenges for the efficient implementation of AI on the edge~\cite{amrouch2022brain}. Recent implementations of HDC on MCU-class devices have demonstrated notable advantages over neural networks. Specifically, HDC leverages simple bitwise operations to achieve low-latency inference and online learning on resource‐constrained platforms, thereby reducing both energy consumption and memory footprint compared to neural network-based approaches~\cite{redding2023embhd}. Moreover, studies have highlighted HDC's inherent robustness to noise and its capacity for lifelong adaptation, making it particularly effective for dynamic edge applications such as gesture, image, and speech recognition~\cite{imani2017voicehd,benatti2019online,peitzsch2024putting}. Unlike CNN, which often struggle to generalize on non-image-based data and rely on memory-intensive nonlinear activations and skip connections, HDC exhibits greater versatility across diverse modalities, rendering it well-suited for multi-task applications on ultra-low power devices.

Despite these advantages, deploying models on MCUs remains constrained by their inherent 8-bit computation limitation, which restricts the effectiveness of quantization techniques that depend on arbitrarily small bitwidths. As a result, prior methods such as QuantHD~\cite{imani2019quanthd}, which rely on fine-grained quantization, fail to fully leverage the benefits of reduced precision since computations on MCUs still default to 8-bit arithmetic. DPQ-HD optimizes HDC models through hardware-aware quantization, leveraging efficient bit-packing on microcontrollers. Additionally, DPQ-HD combines pruning and decomposition to reduce memory and MACs. By adapting to ultra-low-power device constraints, it enables efficient AI deployment on MCUs, paving the way for accessible edge AI solutions.

  

\subsection{Adaptive Online Inference Optimization}
In addition to the offline memory and compute optimizations introduced by DPQ-HD, we propose an online inference optimization strategy to further accelerate the inference process. This optimization is structured into two key phases: a \textit{calibration phase}, which determines the early-exit threshold, and an \textit{adaptive inference phase}, where classification is performed progressively based on the computed confidence margin and early exit mechanisms~\cite{rahmath2024early,chen2024bitwise}.

\subsubsection{Calibration Phase}
A small calibration set (also used for pruning and rank selection) is used to set the early-exit threshold \( \tau \). For each sample, we compute cosine margins between the top two classes and set \( \tau \) as the mean margin. At inference, if a sample’s margin exceeds \( \tau \), evaluation stops early, improving efficiency with minimal accuracy drop.

\subsubsection{Adaptive Inference}  
Inference is performed incrementally by processing hypervectors in fixed chunks of size \( L = \lceil D / C \rceil \), where \( D \) is the total number of dimensions and \( C \) the number of classes. At each step, cosine similarity is computed using partial hypervectors, and the least similar class is eliminated, reducing comparisons while preserving accuracy. To accelerate the process, two classes are removed per iteration until 50\% of the classes are eliminated, as low-probability candidates are unlikely to be strong contenders. Thereafter, a single class is removed per iteration to refine classification. Early exit is considered after 50\% of the classes are eliminated but is only triggered if the confidence margin between the predicted and second-best class exceeds \( \tau \), further reducing computations while maintaining performance. The complete process is summarized in Algorithm~\ref{alg:earlyexit}.

\begin{algorithm}[t]
\caption{Adaptive Inference Strategy}
\label{alg:earlyexit}
\begin{algorithmic}[1]
\Require Sample $s \in \mathbb{R}^D$, Class HVs $H \in \mathbb{R}^{C \times D}$, Threshold $\tau$
\Ensure Predicted class $c^*$
\State $A \gets \{0,\dots,C{-}1\},\; z \gets \mathbf{0} \in \mathbb{R}^C,\; d \gets 0$
\State $n_s \gets \|s\|,\; n_i \gets \|H_i\|\; \forall i,\; L \gets \lceil D/C \rceil$
\While{$|A| > 2$ and $d < D$}
  \State $\ell \gets \min(L, D{-}d)$
  \ForAll{$i \in A$} 
    \State $z[i] \mathrel{+}= \sum_{j=d}^{d+\ell-1} s_j \cdot H_{i,j}$
  \EndFor
  \State $c_i \gets z[i]/(n_s n_i),\; \forall i \in A$
  \If{$|A| > C/2$}
    \State Remove 2 lowest $c_i$ from $A$
  \Else
    \State Remove 1 lowest $c_i$ from $A$
    \If{$|A| \le C/2$ and $c_1 - c_2 \ge \tau$ for top-$2$ $c_i$} \textbf{break} \EndIf
  \EndIf
  \State $d \gets d + \ell$
\EndWhile
\State \Return $\arg\max_{i \in A} c_i$
\end{algorithmic}
\end{algorithm}

%% file: sections/results.tex
\section{Experimental Analysis}\label{sec:results}

\begin{figure*}[t]
   \centering
   \scalebox{0.80}{
   \begin{minipage}[t]{1\textwidth}
      \centering
      \includegraphics[width=\linewidth]{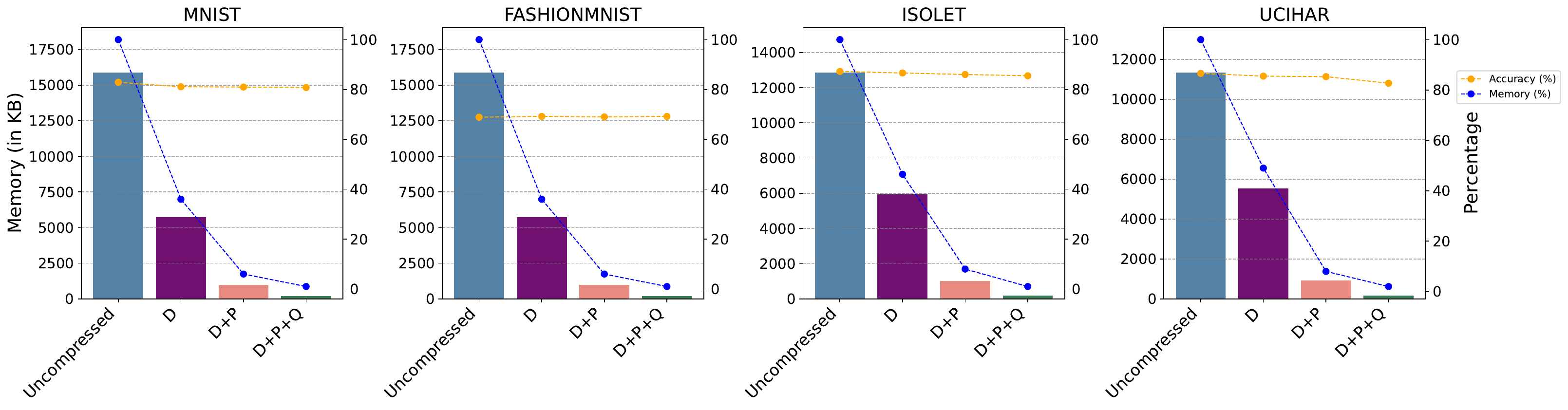}
      \subcaption{CentroidHD (Rank = 256, Prune = 70\%, Bitwidth = 3)}
      \label{fig:top_image}
   \end{minipage}
   }
   
   \vspace{0.1cm} 
   
   \scalebox{0.80}{
   \begin{minipage}[t]{0.55\textwidth}
      \centering
      \includegraphics[width=\linewidth]{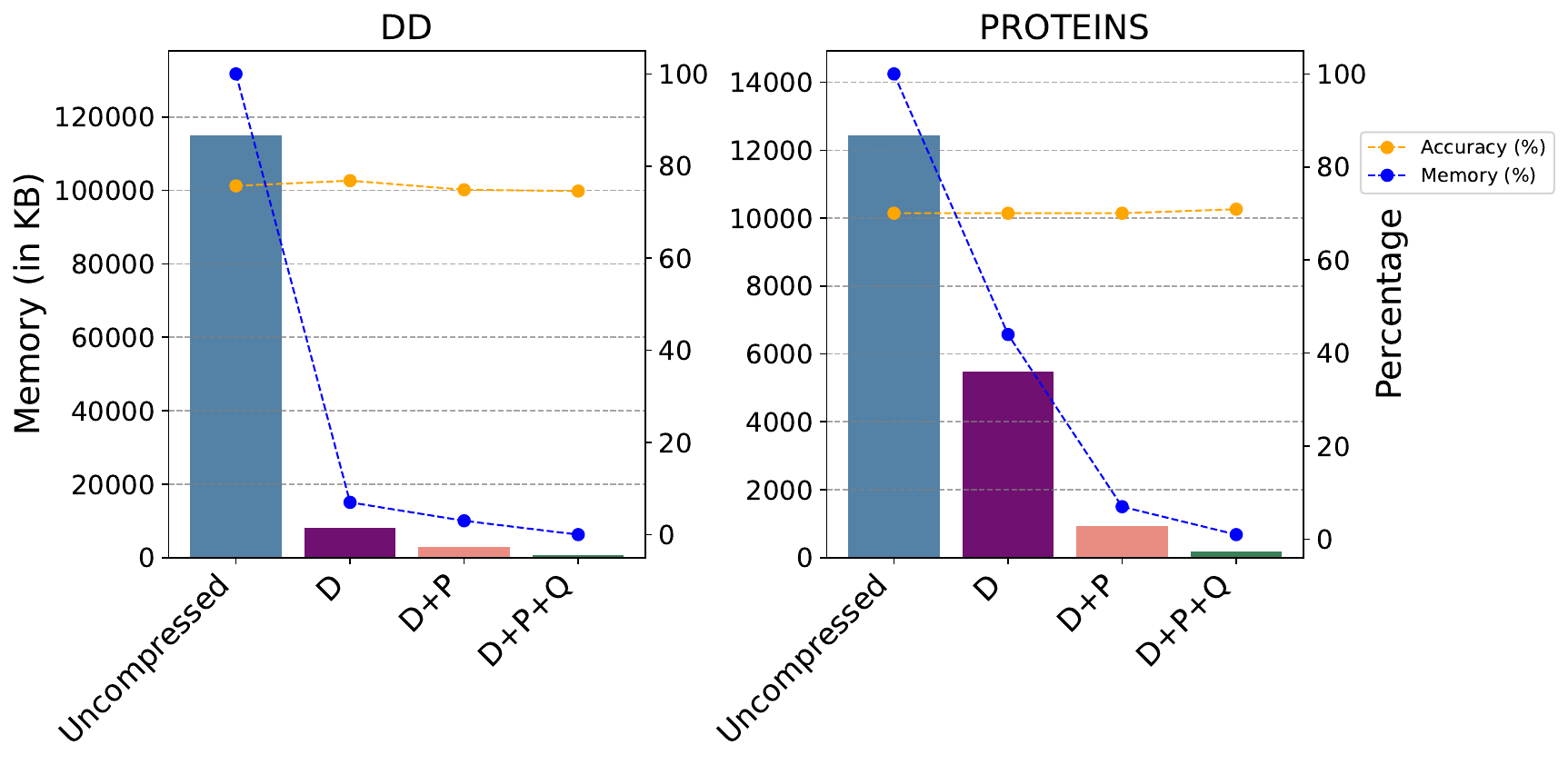}
      \subcaption{GraphHD (Rank = 256, Prune = 40\%, Bitwidth = 3)}
      \label{fig:left_image}
   \end{minipage}
   \hspace{0.5cm} 
   \begin{minipage}[t]{0.34\textwidth}
      \centering
      \includegraphics[width=\linewidth]{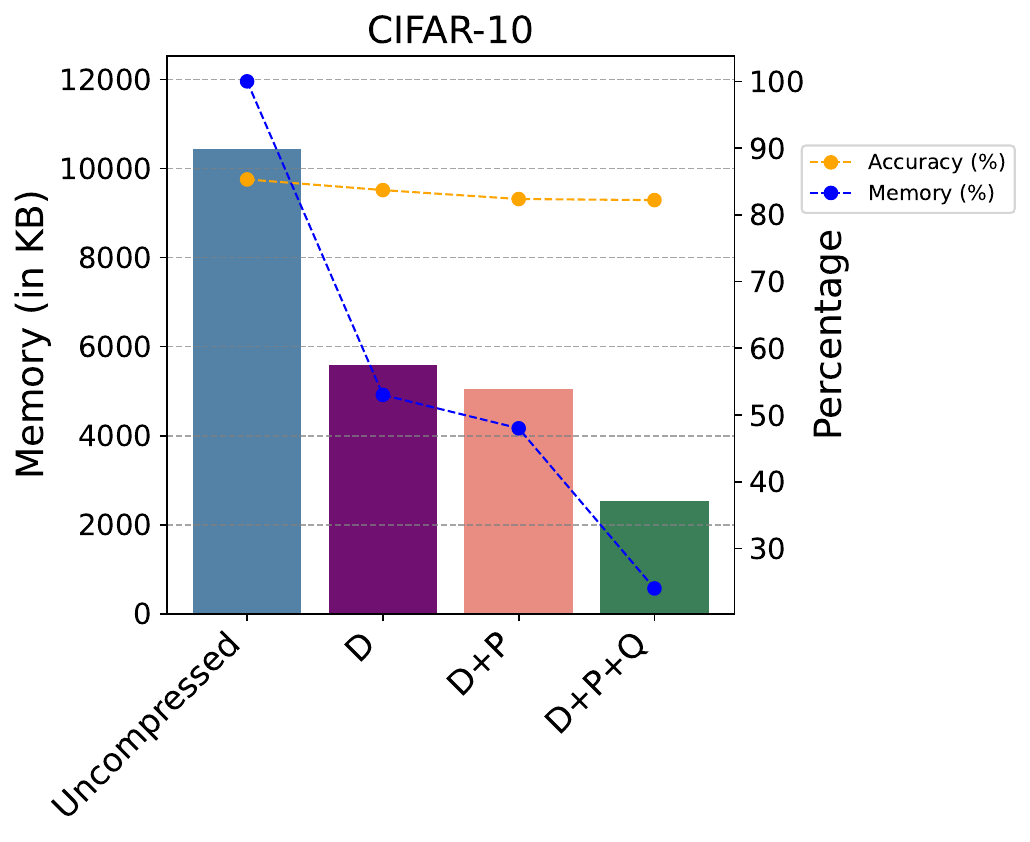}
      \subcaption{HDnn (Rank = 512, Prune = 10\%, Bitwidth = 4)}
      \label{fig:right_image}
   \end{minipage}
   }
   
   \caption{Comparison of uncompressed HDC workloads trained using (a) CentroidHD, (b) GraphHD, and (c) HDnn, and their compressed versions using DPQ-HD. Plots show the impact of decomposition, pruning, and quantization, with the left y-axis showing total memory and the right y-axis showing accuracy and memory reduction relative to the uncompressed model.}
   \label{fig:combined_figure}
\end{figure*}

\subsection{\textbf{Experimental Setup and Baselines}}
To demonstrate the effectivenss of DPQ-HD, we compare it against state-of-the-art (SOTA) methods across different categories. We categorize baselines into three groups: (1) \textit{Task-Specific SOTA}, which represents high-performing models specifically designed for various tasks, (2) \textit{SOTA Compression Baselines}, which include both non-retraining and retraining-based compression techniques, and (3) \textit{Early Exiting Baseline}, which focus on adaptive inference strategies. Each of these baselines provides valuable insights into different aspects of DPQ-HD. 
\subsubsection{\textbf{Datasets}}
We use MNIST~\cite{deng2012mnist}, FASHION-MNIST~\cite{xiao2017fashion}, and CIFAR10~\cite{krizhevsky2010cifar} for image classification, ISOLET~\cite{asuncion2007uci} for speech classification, as well as PROTEINS~\cite{dobson2003distinguishing} and DD~\cite{dobson2003distinguishing} for graph classification. These datasets are chosen as they serve as standard benchmarks for state-of-the-art methods, enabling a direct comparison with both non-retraining and re-training-based baselines.

\subsubsection{\textbf{Task-Specific SOTA}} To demonstrate the versatility of DPQ-HD, we evaluate its performance on diverse task-specific SOTA HDC workloads, using CentroidHD~\cite{kleyko2022survey} for traditional classification, GraphHD~\cite{nunes2022graphhd} for graph-based tasks, and HDNN~\cite{dutta2022hdnn} for large-scale image classification like CIFAR-10~\cite{krizhevsky2010cifar}.

\subsubsection{\textbf{Baselines for SOTA comparison}} We compare our proposed method to available non-retraining-based methods, including the baseline precision-reduction method and Energy-Efficient Sparse Hyperdimensional Computing for speech recognition (Eff-SparseHD)~\cite{buelagala2023energy} which applies redundancy pruning without retraining. We also compare to re-training based compression baselines such as QuantHD~\cite{imani2019quanthd}, MicroHD~\cite{ponzina2024microhd} and DeMAT~\cite{xu2023fsl}. Furthermore, to highlight the effectiveness of our adaptive inference, we compare it against the early-exit baseline BAET~\cite{chen2024bitwise}.

\subsubsection{\textbf{Hardware}} We run our performance and energy evaluation on the Arduino UNO board, featuring an ATmega328P MCU operating at 16 MHz and equipped with 32 KB of flash memory and 2 KB of SRAM. Due to limited memory, we use bit packing to reduce storage and account for bit unpacking overhead at runtime. Since SRAM cannot store all class HVs, cosine similarity is computed iteratively by loading subsets of HDC dimensions and accumulating dot products.

\subsection{Choosing optimum Decomposition Rank and Pruning Ratio}
To achieve efficient compression while maintaining accuracy, we obtain the decomposition rank and pruning ratio using a calibration phase on small validation subsets (e.g., 128 samples). As shown in Figure~\ref{fig:decomposition_pruning_comparison}, increasing the rank improves accuracy up to an optimal point, beyond which redundancy increases (Figure~\ref{fig:decomposition_pruning_comparison}a). Similarly, pruning affects different datasets uniquely: simpler tasks like Fashion-MNIST tolerate higher pruning, while complex tasks like DD degrade significantly with aggressive pruning (Figures~\ref{fig:decomposition_pruning_comparison}b-i, b-ii). By selecting the optimal rank and pruning ratio through this calibration process, we balance computational efficiency and model accuracy, making this approach adaptable to diverse datasets.

\begin{figure*}[h]
   \centering
   \hspace{0.3in}
   \scalebox{0.94}{
   \begin{minipage}[t]{0.32\textwidth}
      \centering
      \includegraphics[width=\linewidth]{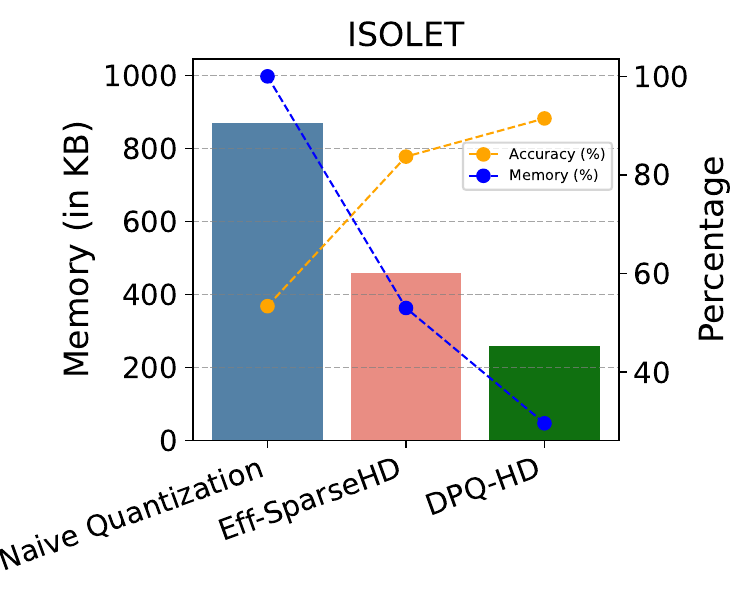}
      \subcaption{Comparison with post-training baselines}
      \label{fig:SOTA-comparisona}
   \end{minipage}
   }
    \scalebox{1.1}{
   \begin{minipage}[t]{0.55\textwidth}
      \centering
      \includegraphics[width=\linewidth]{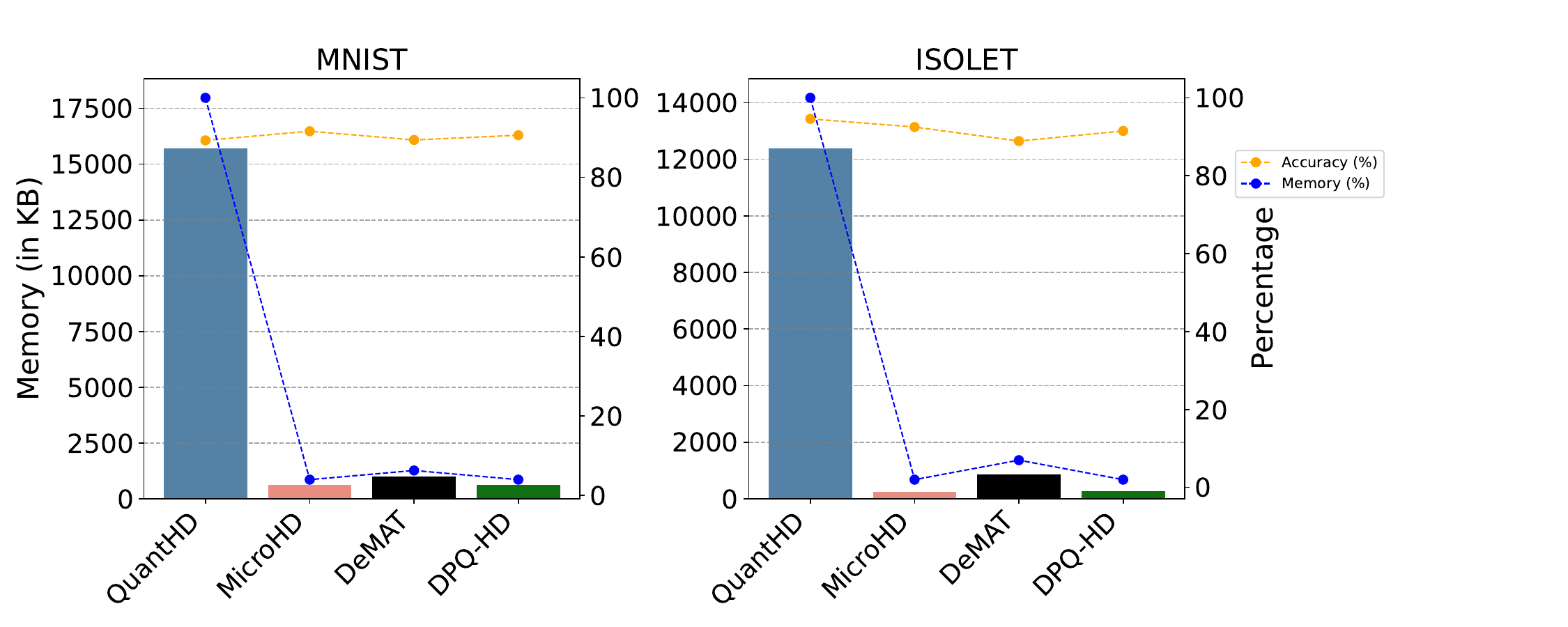}
      \subcaption{Comparison with re-training baselines}
      \label{fig:SOTA-comparisonb}
   \end{minipage}   
    }
   \caption{Comparison of accuracy and memory overhead for (a) post-training baselines: Naive Quantization, Eff-SparseHD~\cite{buelagala2023energy}, and (b) retraining baselines: QuantHD~\cite{imani2019quanthd}, MicroHD~\cite{ponzina2024microhd}, and DeMAT~\cite{xu2023fsl}, trained via OnlineHD~\cite{hernandez2021onlinehd} and compressed with DPQ-HD. The left y-axis represents total memory (encoder + HDC model), while the right y-axis shows accuracy and memory reduction relative to the uncompressed model.}
   \label{fig:SOTA-comparison}
\end{figure*}

\vspace{-0.2cm}
\subsection{Experimental Results} \label{sec-results}
\subsubsection{Generalizability of DPQ-HD Across Task-Specific SOTA}

Figure \ref{fig:combined_figure} highlights the cumulative contribution of each compression technique in DPQ-HD, progressively implementing decomposition, pruning, and quantization on different task specific SOTA HDC baselines. As shown in the figure, decomposition (D) is applied first, followed by pruning (D+P), and finally quantization (D+P+Q), showing cumulative memory reductions achieved at each stage with minimal impact on accuracy. Notably, DPQ-HD after applying all three techniques achieves substantial memory reductions (up to 20-100$\times$ memory reductions compared to uncompressed HDC workloads) with only a 1-2\% loss in accuracy across various datasets, showcasing its ability to drastically reduce memory usage with minimal impact on task performance.  
\vspace{-0.1cm}
\subsubsection{Comparison with non-retraining based compression SOTA}
In order to demonstrate the effectiveness of DPQ-HD, we report comparison of our method with existing baselines on common datasets. We compare DPQ-HD with other approaches that compress HD workloads in a post-training manner. As shown in Figure~\ref{fig:SOTA-comparisona}, DPQ-HD achieves 91.5\% classification accuracy while naive quantization and Eff-SparseHD~\cite{buelagala2023energy} obtain 53.37\% and 83.7\% respectively on the ISOLET~\cite{asuncion2007uci} dataset. By using multiple techniques, DPQ-HD avoids relying solely on pruning, which can cause significant accuracy loss when applied extensively without retraining. This approach enables DPQ-HD to preserve high accuracy even after compression. 

\subsubsection{Comparison with retraining based compression SOTA}
We now compare HDC workloads trained using onlineHD~\cite{hernandez2021onlinehd} and compressed by DPQ-HD and with state-of-the-art based compression techniques, like MicroHD~\cite{ponzina2024microhd}, QuantHD~\cite{imani2019quanthd}, and DeMAT~\cite{xu2023fsl}. Unlike DPQ-HD, these methods follow the 30 or more epochs of retraining for obtaining the final compressed HDC models. As shown in Figure \ref{fig:SOTA-comparisonb}, DPQ-HD achieves 90.61\% and 91.46\% on MNIST and ISOLET respectively and outperforms DeMAT~\cite{xu2023fsl} on both the datasets by obtaining more compression. Also, DPQ-HD achieves models with comparable accuracy and similar size to MicroHD~\cite{ponzina2024microhd}, with accuracy slightly below MicroHD's 91.57\% on MNIST and 92.51\% on ISOLET, all without the retraining overhead. QuantHD~\cite{imani2019quanthd}, on the other hand, performs lower than both DPQ-HD and MicroHD~\cite{ponzina2024microhd} on MNIST with 89.28\% accuracy, and higher on ISOLET with 94.6\% accuracy. However, in both cases, it requires up to 25-50x more memory for the HDC workload, resulting in significant on-device overhead, which is a critical limitation for deployment on microcontroller devices.
\begin{figure}[t]
   \centering
   \scalebox{0.9}{
   \includegraphics[width=0.80\linewidth]{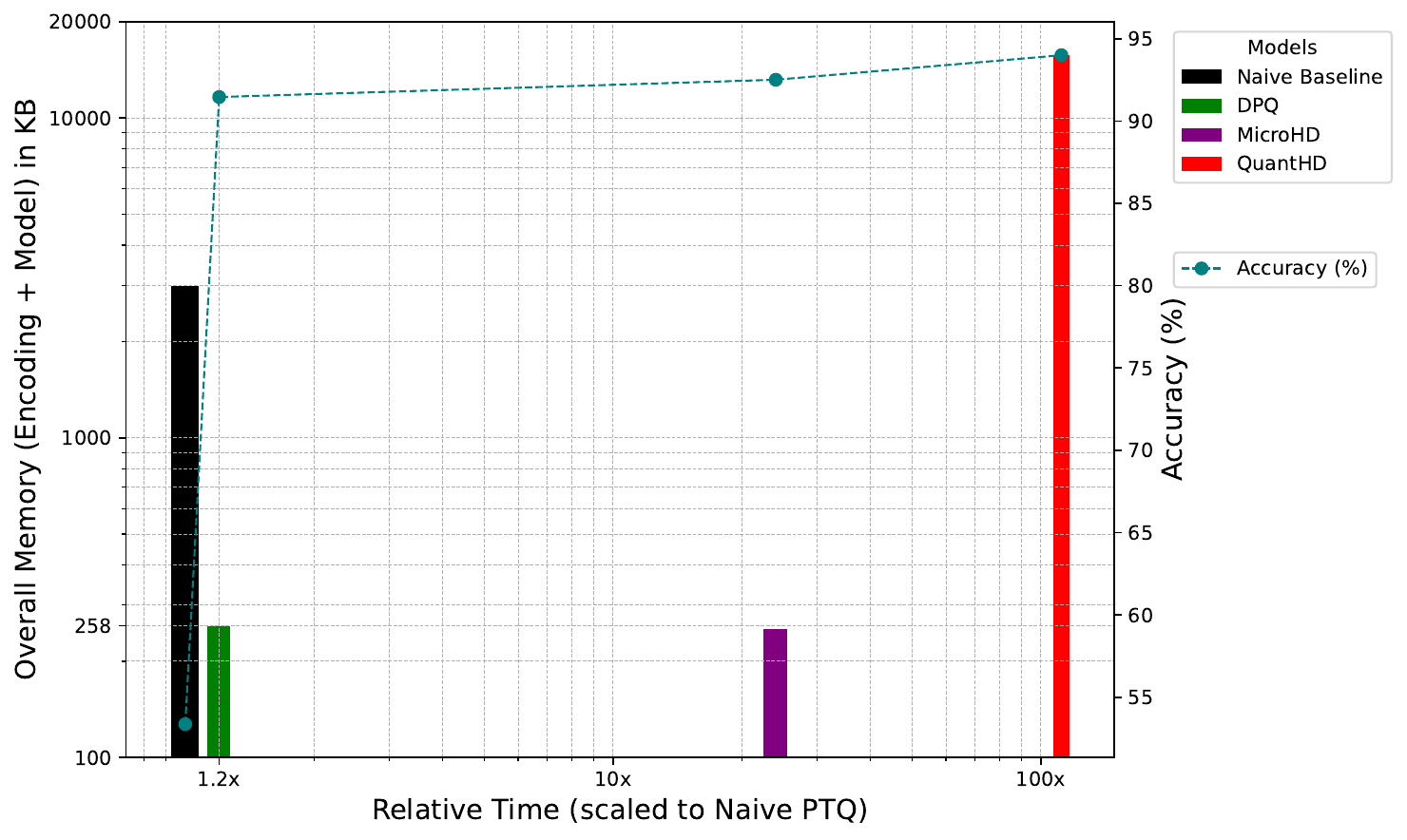} }
   \caption{Comparison of DPQ-HD with MicroHD~\cite{ponzina2024microhd} and QuantHD~\cite{imani2019quanthd} in memory usage, compressed model accuracy, and offline optimization time related to naive quantization.}

   \label{fig:compression_methods}
\end{figure}
\subsection{Comparison of Optimization Time}
As shown in Figure~\ref{fig:compression_methods}, the optimization times required by various methods to obtain compressed HDC workloads differ significantly. MicroHD~\cite{ponzina2024microhd} and QuantHD~\cite{imani2019quanthd}, in particular, require 24x to 100x more time than DPQ-HD due to the retraining phase, often over 30 epochs, to recover compressed model performance. The time comparison reflects only the retraining phase, excluding the additional overhead of initial training or parameter search, further highlighting the inefficiency of retraining-based methods. In contrast, DPQ-HD’s retraining-free approach ensures significantly faster compression without significant accuracy degradation, making it ideal for rapid deployment on memory-constrained edge devices.

\subsection{Power and Inference Latency Analysis}
Lastly, to demonstrate the effectiveness of our framework for on-device deployment on MCU, we herein compare the inference runtime performance of the proposed DPQ-HD, QuantHD~\cite{imani2019quanthd}, MicroHD~\cite{ponzina2024microhd} and DeMAT~\cite{xu2023fsl} using baseline 10k-dimensional HDC implementations as a baseline.
MCUs are often constrained to 8-bit computations, making decomposition and pruning extremely crucial for achieving optimal end-to-end runtime.

As detailed in Table \ref{table:runtime_comparison}, DeMAT~\cite{xu2023fsl}, MicroHD~\cite{ponzina2024microhd} and DPQ-HD demonstrate significant inference performance improvements, resulting in 16.12$\times$, 16.27$\times$ and 56$\times$ respectively, respectively, compared to QuantHD's~\cite{imani2019quanthd} overall speedup of 1.3$\times$ over the baseline. The inefficiency of QuantHD~\cite{imani2019quanthd} on a microcontroller stems from only relying on quantization without leveraging pruning or decomposition techniques. A notable observation is that while MicroHD~\cite{ponzina2024microhd} achieves a similar model size and slightly better accuracy than DPQ-HD, it remains significantly slower on the target hardware. This slowdown is due to MicroHD using 10-bit representations, which are derived from its search-based configuration. In contrast, DPQ-HD efficiently integrates pruning and decomposition while aligning with hardware constraints, leading to a significantly higher speedup (up to 56×).

\begin{table}[t]
    \addtolength{\tabcolsep}{-2pt}
    \centering
    \fontsize{6.8pt}{10.50pt}\selectfont
    \begin{tabular} {l|c|c|c|c|c}
    \toprule
     & Baseline & QuantHD~\cite{imani2019quanthd}  & DeMAT~\cite{xu2023fsl} & MicroHD~\cite{ponzina2024microhd} & DPQ-HD (ours) \\
     \midrule
     Runtime (s) & 17.90 & 13.80 & 1.11 & 1.10 & 0.32 \\
     Energy (mJ) & 282  & 218 & 17 & 16.84 & 5.05 \\
     Improvement & 1$\times$ & 1.3$\times$ & 16.12$\times$ & 16.27 $\times$& 56$\times$ \\

    \midrule

    \end{tabular}
        \caption{Inference performance and energy evaluation on the ATmega328P MCU after different optimization approaches.}

    \label{table:runtime_comparison}
\end{table}

\begin{table}[h]
    \centering
    \fontsize{7.5pt}{10.50pt}\selectfont
    \begin{tabular}{l|cc}
        \hline
        \textbf{Method} & \textbf{MNIST} & \textbf{ISOLET} \\
        \hline
        BAET~\cite{chen2024bitwise} & 69.5\% & 70.1\% \\
        Ours & \textbf{76.02\%} & \textbf{76.94\%} \\
        \hline
    \end{tabular}
    \caption{Comparison of prediction runtime reduction while maintaining original accuracy.}
    \label{table:ops_reduction}
\end{table}
\vspace{-0.8cm}
We compare our adaptive online inference strategy with the state-of-the-art early-exit baseline BAET~\cite{chen2024bitwise}. As shown in Table~\ref{table:ops_reduction}, our method reduces MNIST and ISOLET runtime by 76.02\% and 76.94\%, outperforming BAET’s 69.5\% and 70.1\%, while maintaining accuracy. This demonstrates the efficiency of our approach without compromising performance.


%% file: sections/conclusions.tex
\section{Conclusion}\label{sec:conclusions}
In this work, we introduced DPQ-HD, a novel post-training compression algorithm designed to compress end-to-end HDC workloads while maintaining close to the uncompressed performance without retraining. Our extensive experiments across various datasets show that DPQ-HD achieves memory reductions of up to 20$\times$ for image classification and 100$\times$ for graph classification tasks, with only a minimal 1-2\% drop in accuracy compared to uncompressed HD workloads, highlighting DPQ-HD's effectiveness for edge suitability. Furthermore, we show that DPQ-HD outperforms existing post-training pruning baselines in classification accuracy and achieves performance comparable to retraining-based state-of-the-art methods, all while requiring significantly less optimization time (up to 100$\times$) and offering significantly faster inference and lower power consumption (up to 56$\times$) on a microcontroller. Additionally, our adaptive inference strategy dynamically adjusts computation, progressively refining predictions and eliminating unlikely classes. Our approach reduces bitwise operations during prediction by up to 76.94\% while preserving accuracy. DPQ-HD, along with the added benefits of adaptive inference, provides an efficient framework for edge AI, enabling fast, low-power HDC deployment.

%% file: sections/acknowledgement.tex
\section{Acknowledgments}
This work has been funded in part by NSF, with award numbers \#1826967, \#1911095, \#2003279, \#2052809, \#2100237, \#2112167, \#2112665, and in part by PRISM and CoCoSys, centers in JUMP 2.0, an SRC program sponsored by DARPA.